\documentclass[11pt, oneside]{article}   	
\usepackage[margin=1.0in]{geometry}                		
\usepackage{graphicx}				
\usepackage{amssymb}
\usepackage{amsmath}
\usepackage{amsthm}
\usepackage{mathtools}
\usepackage{hyperref}
\usepackage{enumitem}
\usepackage{verbatim}
\usepackage{xcolor}
\usepackage[utf8]{inputenc}
\usepackage[T1]{fontenc}
\usepackage{microtype}
\usepackage{natbib}

\newtheorem{theorem}{Theorem}
\newtheorem{definition}{Definition}[section]

\newtheorem{lemma}{Lemma}
\newtheorem*{lemma*}{Lemma}
\newtheorem{claim}{Claim}[theorem]

\newcommand{\E}{\mathop{\mathbb{E}}}

\newcommand{\R}{\mathbb{R}}
\newcommand{\eps}{\epsilon}

\let\originalleft\left
\let\originalright\right
\renewcommand{\left}{\mathopen{}\mathclose\bgroup\originalleft}
\renewcommand{\right}{\aftergroup\egroup\originalright}

\renewcommand{\epsilon}{\varepsilon}

\usepackage{bbm}
\usepackage{float}
\usepackage[]{algorithm}
\usepackage[noend]{algpseudocode}
\makeatletter
\def\BState{\State\hskip-\ALG@thistlm}
\makeatother
\title{Efficient Private Algorithms for Learning Large-Margin Halfspaces}
\author{Huy L\^{e} Nguy\~{\^{e}}n\thanks{Khoury College of Computer Sciences, Northeastern University. Supported by NSF CAREER CCF-1750716. \href{mailto:hlnguyen@cs.princeton.edu}{\texttt{hlnguyen@cs.princeton.edu}}.} \and Jonathan Ullman\thanks{Khoury College of Computer Sciences, Northeastern University. Supported by NSF grants CCF-1718088, CAREER CCF-1750640, and CNS-1816028, and a Google Faculty Research Award. \href{mailto:jullman@ccs.neu.edu}{\texttt{jullman@ccs.neu.edu}}.} \and Lydia Zakynthinou\thanks{Khoury College of Computer Sciences, Northeastern University. Partly supported by a Graduate Fellowship from the Khoury College of Computer Sciences. \href{mailto:zakynthinou.l@northeastern.edu}{\texttt{zakynthinou.l@northeastern.edu}}.}}

\begin{document}

\maketitle

\begin{abstract}%
We present new differentially private algorithms for learning a large-margin halfspace. 
In contrast to previous algorithms, which are based on either differentially private simulations of the statistical query model or on private convex optimization, the sample complexity of our algorithms depends only on the margin of the data, and not on the dimension. 
We complement our results with a lower bound, showing that the dependence of our upper bounds on the margin is optimal.
\end{abstract}

\section{Introduction}
In a classification problem, we are given labeled examples from some unknown distribution, and the goal is to learn a classifier that accurately labels future examples from the same distribution.  In many applications, each of these examples represents the highly sensitive \emph{privacy} information of some individual.  Although the goal of classification is to learn about the distribution, and not about the examples {\em per se}, many natural learning algorithms have the unfortunate side effect of revealing all or part of some of the labeled examples.  For example, support vector machines represent the learned classifier as a set of support vectors, which are just labeled examples from the input!

The now-standard approach for ensuring privacy in machine learning is \emph{differential privacy (DP)}~\citep{DMNS06}, which, informally, requires that no individual labeled example in the input significantly influences the learned classifier. Starting with some of the earliest work in differential privacy~\citep{BDMN05,KLNRS08}, there is a large body of literature showing that nearly every classification problem can be solved with differential privacy, albeit with large overheads in both sample complexity and running time.  It is thus central to understand for which problems these overheads can be eliminated, and for which they are inherent.

In this paper we study the classical problem of learning a \emph{large-margin halfspace}.  That is, the examples are unit vectors $\mathbf{x} \in \mathbb{R}^d$ labeled with $y \in \{\pm1\}$, and we assume that $y = \mathrm{sign}(\langle \mathbf{w}, \mathbf{x} \rangle)$ for some unknown unit vector $\mathbf{w} \in \mathbb{R}^d$.  Further, no example falls too close to the boundary of the halfspace, meaning that
$y \cdot \langle \mathbf{w}, \mathbf{x} \rangle \geq \gamma$, where $\gamma$ is called the \emph{margin}. Any algorithm (private or non-private) for learning halfspaces over arbitrary distributions requires sample complexity growing polynomially in the dimension $d$. When $d$ is large, assuming a large margin enables learning the halfspace with sample complexity independent of $d$.

Many results in differential privacy either explicitly or implicitly give private algorithms for learning a large-margin halfspace (see the related work for a detailed discussion). \citet{BDMN05} gave a differentially private implementation of the classical Perceptron algorithm for learning a large-margin halfspace, however their implementation requires sample complexity $\mathrm{poly}(d)$, which is precisely what the large-margin assumption is meant to avoid.  A distinct line of work, beginning with~\citet{CMS11}, studies differentially private algorithms for empirical loss minimization problems.  Although learning a large-margin halfspace can be achieved via minimizing the \emph{hinge loss}, generic algorithms for differentially private loss minimization inherently require $\mathrm{poly}(d)$ samples~\citep{BUV14,BST14}.

\subsection{Results}

In this work we give two new differentially private algorithms for learning a large-margin halfspace.  The key feature of our algorithms is that the sample complexity depends only on the margin, the desired accuracy of the learner, and the desired level of privacy, and not on the dimension.  More precisely,  our sample complexity is (ignoring constants and logarithmic factors) $1 / \alpha \varepsilon \gamma^2$ where $\alpha$ is the desired error and $\varepsilon$ is the desired privacy.  In contrast, without privacy the sample complexity is roughly $1/\alpha \gamma^2$, so our sample complexity is comparable to that of non-private algorithms except when $\varepsilon$ is very small.

Our first algorithm runs in polynomial time in all the parameters and satisfies the standard notion of $(\varepsilon, \delta)$-DP.  Our second algorithm's running time grows exponentially in  the inverse-margin $1/\gamma$, but the algorithm satisfies the very strong special case of $(\varepsilon,0)$-DP (so-called \emph{pure DP}).  Our results are described in more detail in Table~\ref{table:results}. For simplicity, each of the bounds in Table~\ref{table:results} suppresses polylogarithmic factors of $\alpha,\beta,\varepsilon,\delta,\gamma$.

The main technique in both of our algorithms is to use random projections to reduce the dimensionality of the space to $\approx 1/\gamma^2$.  After projection, we can learn using either a differentially private algorithm for minimizing hinge loss or by using the exponential mechanism over a net of possible halfspaces.  We note that using either of these techniques on its own, without the projection, would fail to find an accurate classifier without $\mathrm{poly}(d)$ samples. We also note that one could apply a random projection and then run the algorithm of~\citet{BDMN05}, which would have sample complexity on the order of $1/{\alpha\eps\gamma^3}$, which would be suboptimal.

We also prove a lower bound showing that any $(\eps,0)$-differentially private algorithm for learning a large-margin halfspace (with constant classification error) requires $\Omega(1/\eps\gamma^2)$ samples (unless $d = o(1/\gamma^2)$).  This lower bound is presented in Theorem~\ref{th:LB}.

\renewcommand{\arraystretch}{1.75}
\begin{table}[t!]
\centering
\begin{tabular}{|l|c|c|c|}
\hline
	& Sample Complexity & Running Time & Privacy       \\
\hline
Theorem~\ref{th:UB_approximate} & $\frac{1}{\alpha\varepsilon\gamma^2}$               &   $\mathrm{poly}(\frac{d\log(1/\beta \delta)}{\alpha \varepsilon \gamma})$        & $(\varepsilon, \delta)$ \\
\hline
Theorem~\ref{th:UB_pure} & $\frac{1}{\alpha\varepsilon\gamma^2}$            &     $(2^{\tilde{O}(1/\gamma^2)} +d) \cdot \mathrm{poly}(\frac{\log(1/\beta)}{\alpha \varepsilon \gamma})$      & $(\varepsilon, 0)$ \\   
\hline
\end{tabular}
\caption{Sample complexity and running time bounds for our algorithms.   For simplicity, each of these bounds suppresses polylogarithmic factors of $\alpha,\beta,\varepsilon,\delta,\gamma$.}
\label{table:results}
\end{table}

\subsection{Related Work}

\citet{BDMN05} gave a differentially private implementation of the classical Perceptron algorithm, based on a general differentially private simulation of algorithms in the statistical queries model~\citep{Kearns93}. Their algorithm can be improved using more recent statistical queries algorithms by~\citet{FGV17}, but this approach still requires $\mathrm{poly}(d)$ samples. The foundational work of~\citet{KLNRS08} studied differentially private PAC learning, and gave a generic private PAC learner, but they did not consider margin-based learning guarantees. In a recent work,~\citet{BMNS19a} give a private learner for halfspaces over an arbitrary finite domain, significantly improving the dependence on the domain size by reducing the problem to the one of privately locating an approximate center point. They prove that their algorithm requires $\mathrm{poly}(d)$ samples and that this technique can not yield a better bound, but again they do not consider the large-margin assumption.

An alternative approach is to leverage algorithms for differentially private convex optimizing to identify a halfspace minimizing the hinge loss.  Differentially private convex optimization is now the subject of a large body of literature that is too large to survey here.  Notably,~\citet{BST14} gave nearly optimal algorithms for private convex optimization in the relevant setting, and showed that such algorithms necessarily require $\mathrm{poly}(d)$ samples. \citet{JainT14} gave nearly dimension-free results for minimizing \emph{generalized linear models}, which could be used for learning a large-margin halfspace. However the sample complexity obtained by using these results would be $1/{\alpha^2\eps^2\gamma^2}$, which is again significantly worse than ours.

\citet{BLR08} gave an algorithm for the related \emph{query release} problem for large-margin halfspaces---they construct a differentially private algorithm that outputs a data structure such that one can input a halfspace such that if the data has large margin with respect to that halfspace, then the structure outputs an estimate of how many points are labeled positively.  One could use such a data structure to learn a large-margin halfspace, however, their algorithm has sample complexity $\mathrm{poly}(d)$, and the resulting learning algorithm would also not be computationally efficient.

Random projections have proven to be a very useful tool in learning theory in applications that assume some kind of separability~\citep{Vempala04, Blum05}). Similar to our work, there have also been applications of random projections in differential privacy.  One example is the above query release algorithm from~\citet{BLR08}, which is conceptually similar to our pure differentially private algorithm.  In a very different setting,~\citet{BBDS13} demonstrated that certain random projection matrices \emph{automatically} preserve privacy, however there is no technical relationship between their results and ours.  \citet{KKMN12} also used the Johnson-Lindenstrauss transform to achieve better utility and computational efficiency for privately estimating distances between users, but again there is no technical relationship between their results and ours.

\section{Preliminaries}
\subsection{Learning Halfspaces}
We consider a distribution $D$ over $\mathcal{X} \times \{\pm 1\}$, where $\mathcal{X} \subseteq \mathbb{R}^d$. We denote by $\mathcal{B}^d_2(r)$ the ball in $\mathbb{R}^d$ with center $\mathbf{0}$ and radius $r$ with respect to the euclidean norm $\| \cdot \|_2$, and $\mathcal{B}^d_2(1)=\mathcal{B}^d_2$. We assume that all examples are normalized so that $\mathcal{X} = \mathcal{B}^d_2$. This assumption is without loss of generality. Since we can normalize each point in the input dataset, this operation would not affect privacy: any two neighboring datasets, as defined in the next section, would remain neighbors. As for utility, the (normalized) margin, as defined in this section, would also remained unchanged. Furthermore, if we consider $\gamma$ to be the non-normalized margin and assume $\mathcal{X}\subseteq\mathcal{B}_2^d(R)$ for some known $R$, we can still run the algorithm with the scaled margin $\gamma/R$, increasing the sample complexity to $\tilde{O}(R^2/\alpha\eps\gamma^2)$.

A linear threshold function is defined as $f_{\mathbf{w},\theta}(\mathbf{x}) = \textup{sign}(\langle \mathbf{w} ,\mathbf{x} \rangle + \theta)$, where $\mathbf{x}, \mathbf{w} \in \mathbb{R}^d$ and $\theta \in \mathbb{R}$. We assume without loss of generality that $\theta=0$ so $f_{\mathbf{w}}(\mathbf{x}) = \textup{sign}(\langle \mathbf{w} ,\mathbf{x} \rangle)$.\footnote{If $m_{\mathcal{X}}=\min_{\mathbf{x}\in\mathcal{X}}\|\mathbf{x}\|_2$ is known, we can run the algorithm with the modified points $\mathbf{\tilde{x}} = [\mathbf{x},1]\in\mathbb{R}^{d+1}$, margin $\tilde{\gamma}=\gamma\cdot\min\{1,m_{\mathcal{X}}\}/(2+2|\theta|)$, and hypothesis space $\{ [\mathbf{w},\theta] : \mathbf{w}\in\mathcal{B}_2^d \}$.}
We call a vector $\mathbf{w}$ a hypothesis. The error of a threshold function defined by hypothesis $\mathbf{w}$ on distribution $D$ is 
\[
\mathrm{err}_{D}(f_{\mathbf{w}}) = \Pr_{(\mathbf{x},y)\sim D}[f_{\mathbf{w}}(\mathbf{x}) \neq y] = \Pr_{(\mathbf{x},y)\sim D}[\textup{sign}(\langle \mathbf{w}, \mathbf{x} \rangle) \neq y] = \Pr_{(\mathbf{x},y)\sim D}[y\cdot \langle \mathbf{w}, \mathbf{x} \rangle < 0].
\]
As in the PAC (\emph{probably approximately correct}) learning model, introduced by~\cite{Valiant84}, the goal is to find a hypothesis $\mathbf{w}$ such that $\mathrm{err}_{D}(f_{\mathbf{w}}) \leq \alpha$ with probability $1-\beta$, for given parameters $\alpha$ and $\beta$. We assume that there exists a hypothesis with zero error, that is, there exists a $\mathbf{w}^*\in \mathcal{B}^d_2$ such that $y \cdot \langle \mathbf{w}^*, \mathbf{x} \rangle > 0$ $\forall (\mathbf{x}, y)$. 
More specifically, we assume that $\mathbf{w}^*$ maximizes the margin 
\[
\gamma = \min_{\mathbf{x}\in \mathcal{X}} \frac{\vert \langle \mathbf{w}^*,\mathbf{x} \rangle \vert }{\|\mathbf{w}^*\|_2 \cdot \|\mathbf{x}\|_2},
\]
 which is assumed to be known in advance. Equivalently, $\gamma \leq \vert  \cos (\mathbf{w}^*, \mathbf{x})\vert$ $\forall \mathbf{x}$, where the right hand side is the distance of a scaled point $\mathbf{x}$ from the halfspace $\langle \mathbf{w}^*, \mathbf{x} \rangle=0$.

Our goal is to design algorithms which, given enough data points drawn from a distribution $D$ over a linearly 
separable set with margin $\gamma$, return a hypothesis which has error at most $\alpha$ with respect to the distribution, with probability $1-\beta$. More formally, we aim to design an $(\alpha,\beta,\gamma)$-PAC learner with low sample complexity.

\begin{definition}[$(\alpha,\beta,\gamma)$-PAC learner]
Let $D$ be a distribution over $\mathcal{B}_2^d \times \{\pm 1\}$ such that there exists $\mathbf{w}^* \in \mathcal{B}_2^d$ for which
$\Pr_{(\mathbf{x},y) \sim D}\left[ y \langle \mathbf{w}^*,\mathbf{x} \rangle \geq \gamma \right] = 1.$ We call such a distribution $D$ a \emph{distribution with margin $\gamma$}.
An algorithm $\mathcal{A}$ is an \emph{$(\alpha,\beta,\gamma)$-PAC learner for halfspaces in $\R^d$ with margin $\gamma$ and sample complexity $n$} if, given a sample set $S \sim D^{n}$ from any distribution $D$ with margin $\gamma$, it outputs a classifier $\mathcal{A}(S) = \mathbf{\hat{w}} \in \mathcal{B}_2^d$ such that, with probability at least $1-\beta$,
\[\Pr_{(\mathbf{x},y) \sim D}\left[ y =  \mathrm{sign}(\langle \mathbf{\hat{w}},\mathbf{x} \rangle) \right] \geq 1-\alpha.\]
\end{definition}

\subsection{Differential Privacy}\label{sec:dp}
We design algorithms which draw a sample set $S$ and output a hypothesis $\mathbf{w}\in \mathbb{R}^d$. In addition to finding a good hypothesis, our algorithms must satisfy \textit{differential privacy (DP)} guarantees. Differential privacy is a property that a randomized algorithm satisfies if its output distribution does not change significantly under the change of a single data point. 

More formally, let $S, S'\in \mathcal{S}^n$ be two data sets of the same size. We say that $S,S'$ are \textit{neighbors}, denoted as $S\sim S'$, if they differ in at most one data point.
\begin{definition}[Differential Privacy,~\cite{DMNS06}]
A randomized algorithm $\mathcal{A}:\mathcal{S}^n \rightarrow \mathcal{O}$ is  {\em $(\epsilon,\delta)$-differentially private} if for all neighboring data sets $S, S'$ and all measurable $O\subseteq \mathcal{O}$,
\[\Pr[\mathcal{A}(S) \in O] \leq \exp(\eps) \Pr[\mathcal{A}(S')\in O] + \delta.\]
Algorithm $\mathcal{A}$ is {\em $(\eps,0)$-differentially private} if it satisfies the definition for $\delta=0$.
\end{definition}

A useful property of differential privacy is that it is closed under post-processing.
\begin{lemma}[Post-Processing,~\cite{DMNS06}]\label{lem:post-processing}
Let $\mathcal{A}:\mathcal{S}^n \rightarrow \mathcal{O}$ be a randomized algorithm that is $(\eps,\delta)$-differentially private.  For every (possibly randomized) $f : \mathcal{O} \to \mathcal{O}'$, $f \circ \mathcal{A}$ is $(\eps,\delta)$-differentially private.
\end{lemma}


\section{An Efficient Private Algorithm}\label{sec:approx}

Both the algorithm of this and the next section draw a sample set $S\sim D^n$ and perform dimension reduction from a $d$-dimensional to an $m$-dimensional space, which allows them to run in the reduced space for the remainder of the execution.

\begin{algorithm}[H]
\caption{$\mathcal{A}_{\alpha,\beta,\eps,\delta,\gamma}(S)$}\label{approximate}
\begin{algorithmic}[1]
\State Choose a random matrix $A\in \mathbb{R}^{m \times d}$, where $m=O\left(\frac{\log(1/\beta_{\mathit{JL}})}{\gamma^2}\right)$, $\beta_{\mathit{JL}}=\alpha\beta^2/64n$, and $$A_{ij} = \begin{cases}
        +1/\sqrt{m} &\text{w.p.~} 1/2\\
        -1/\sqrt{m} &\text{w.p.~} 1/2
        \end{cases}$$\;
\State Define $S_A \gets \left\{\left(A\mathbf{x} / \|A\mathbf{x} \|_2,y\right) \mid (\mathbf{x},y)\in S\right\}$.\;
\State Define the hypothesis set $\mathcal{C} \gets \mathcal{B}_2^m$.\;
\State Define the $\frac{100}{86\gamma}$-Lipschitz loss function $\ell : \mathcal{C}\times (\mathcal{B}^m_2\times\{\pm1\}) \to \mathbb{R}$ as $$\ell(\mathbf{w};(\mathbf{x},y)) =\mathbbm{1}{\left\{y\cdot \langle \mathbf{w}, \mathbf{x}\rangle < \frac{96\gamma}{100}\right\}} \cdot \left( \frac{96}{86}-\frac{y\cdot \langle \mathbf{w}, \mathbf{x}\rangle}{86\gamma/100}\right).$$\;
\State Let $\mathbf{\hat{w}} \gets \mathcal{F}(S_A, \ell, (\eps,\delta), \mathcal{C})$.\;
\State Return $\mathbf{\hat{w}}^\top A$.\;
\end{algorithmic}
\end{algorithm}

Algorithm $\mathcal{F}$ is any differentially private empirical risk minimization algorithm. We can instantiate it with the noisy stochastic gradient descent algorithm of~\cite{BST14} so that it has the following guarantee. The full algorithm is presented in Appendix~\ref{app:algo} for completeness.
\begin{theorem}[\cite{BST14}]\label{th:BST}
Let sample set $D$, $L$-Lipschitz loss function $\ell$, differential privacy parameters $(\eps, \delta)$, and convex hypothesis space $\mathcal{C}$ with diameter $\|\mathcal{C}\|_2$. There exists $(\eps, \delta)$-differentially private algorithm $\mathcal{F}$, such that with probability $1-\beta/4$, its returned hypothesis $\mathbf{\hat{w}}$ satisfies
\begin{equation}\label{eq:BST}
\mathcal{L}(\mathbf{\hat{w}}; D) - \min\limits_{\mathbf{w}\in\mathcal{C} \subseteq \R^{m}} \mathcal{L}(\mathbf{w}; D) = \frac{\sqrt{m} L \|\mathcal{C}\|_2}{\eps} \cdot \mathrm{polylog}\left(n, \frac{1}{\beta},\frac{1}{\delta} \right),
\end{equation}
where $\mathcal{L}(\mathbf{w}; D) = \sum_{(\mathbf{x},y)\in D}\ell(\mathbf{w};(\mathbf{x},y))$ is the total loss of a hypothesis $\mathbf{w}$ on the data set $D$.
\end{theorem}

For the following proofs, we denote $\mathbf{x}_A := \frac{A\mathbf{x}}{\|A\mathbf{x}\|_2}$ for any $\mathbf{x}\in\mathcal{B}_2^d$. It holds that $\mathbf{x}_A \in\mathcal{B}_2^m$ and the modified sample set can be also written as $S_A = \left\{(\mathbf{x}_A,y) \mid (\mathbf{x},y)\in S\right\}$.
The lemma that follows guarantees that the transformation of a point $\mathbf{x} \mapsto A\mathbf{x}$, with high probability, only changes its euclidean norm by a small multiplicative factor. 


\begin{lemma}[Distributional Johnson-Lindenstrauss Lemma,~\cite{optas03}]\label{lemma:DJL}
Let $A\in \mathbb{R}^{m \times d}$ be a random matrix such that $m=O\left(\frac{\log(1/\beta_{\mathit{JL}})}{\gamma^2}\right)$ and $A_{ij} = \begin{cases}
        +1/\sqrt{m}  &\text{w.p.~$1/2$}\\
        -1/\sqrt{m} &\text{w.p.~$1/2$}
        \end{cases}.$
Then, for every $\mathbf{x}\in \mathbb{R}^d$, it holds that $\Pr_A\left[\left\vert \|A\mathbf{x}\|^2_2 - \|\mathbf{x}\|^2_2\right\vert \leq \frac{\gamma}{100} \|\mathbf{x}\|^2_2\right] \geq 1-\beta_{\mathit{JL}}.$
\end{lemma}

Since the transform leaves the norm of a point $\mathbf{x}$ almost unchanged, one would expect that the corresponding transformed and normalized hypothesis $\mathbf{w}_A^* := A\mathbf{w}^*/\|A\mathbf{w}^*\|_2$ would still have a large enough margin with respect to the corresponding point $\mathbf{x}_A$. The following lemma defines the probability that a point $\mathbf{x}$ belongs in the set of points $\mathcal{G}_A$, which are ``good'' for a fixed matrix $A$, in the sense that their norm remains almost unchanged and the margin of their corresponding points $\mathbf{x}_A$ from $\mathbf{w}_A^*$ is close to the original.

\begin{lemma}\label{lem:G_A}
For every given matrix $A$, let $\mathcal{G}_A \subseteq \mathcal{X}\times\{\pm1\}$ be the set of data points $(\mathbf{x},y)$ that satisfy the following two statements:
\begin{enumerate}[label=(\roman*)]
\item $\left\vert \|A\mathbf{x}\|^2_2 - \|\mathbf{x}\|^2_2\right\vert \leq \frac{\gamma}{100} \|\mathbf{x}\|^2_2$ and 
\item $\mathbf{w}_A^* = \frac{A\mathbf{w}^*}{\|A\mathbf{w}^* \|_2}$ has margin at least $96\gamma/100$ on $(\mathbf{x}_A,y)$, i.e.~$y \cdot \langle \mathbf{w}^*_A, \mathbf{x}_A\rangle \geq 96\gamma/100$.
\end{enumerate}
It holds that $\Pr_{(\mathbf{x},y)\sim D}[(\mathbf{x},y)\in \mathcal{G}_A] \geq 1-4\beta_{\mathit{JL}}.$
\end{lemma}
For the proof of Lemma~\ref{lem:G_A}, we express an inner product as $\langle \mathbf{w}^*, \mathbf{x}\rangle = \frac{1}{4}\|\ \mathbf{x}+\mathbf{w}^*\|^2_2 - \frac{1}{4}\|\ \mathbf{x}-\mathbf{w}^*\|^2_2 \\$ and use the guarantee of Lemma~\ref{lemma:DJL} on vectors $\mathbf{x}, \mathbf{w}^*, \mathbf{x}-\mathbf{w}^*, \mathbf{x}+\mathbf{w}^*$. By union bound, with probability $1-4\beta_{\mathit{JL}}$, $\mathbf{x}_A$ has margin $96\gamma/100$ with respect to $\mathbf{w}_A^*$. 
The proof of Lemma~\ref{lem:G_A} follows by four applications of Lemma~\ref{lemma:DJL} and is in Appendix~\ref{app:algo}.

In the following, we provide the privacy and sample complexity guarantees of our algorithm. 

\begin{theorem}[Sample complexity]\label{th:UB_approximate}
Algorithm $\mathcal{A}_{\alpha,\beta,\eps,\delta, \gamma}$ is an $(\alpha, \beta, \gamma)$-learner with sample complexity \[n=\frac{1}{\alpha\eps\gamma^2}\cdot \mathrm{polylog}\left(\frac{1}{\alpha}, \frac{1}{\beta}, \frac{1}{\delta}, \frac{1}{\eps}, \frac{1}{\gamma}\right).\]
\end{theorem}

\begin{proof}[Proof of Theorem \ref{th:UB_approximate}]
The first step of the algorithm is to sample matrix $A$ uniformly at random from $U=\left\{\pm\frac{1}{\sqrt{m}}\right\}^{m\times d}$. From Lemma~\ref{lem:G_A}, it follows that  $
\E\limits_{A}\left[\Pr\limits_{(\mathbf{x},y)\sim D}[(\mathbf{x},y)\notin \mathcal{G}_A]\right] \leq 4\beta_{\mathit{JL}}.$
And, by  Markov's inequality,
\begin{align*}
\Pr\limits_A\left[ \Pr\limits_{(\mathbf{x},y)\sim D}[(\mathbf{x},y)\notin \mathcal{G}_A] \geq \beta'\right] &\leq \frac{\E\limits_{A}\left[\Pr\limits_{(\mathbf{x},y)\sim D}[(\mathbf{x},y)\notin \mathcal{G}_A]\right]}{\beta'}
\leq \frac{4\beta_{\mathit{JL}}}{\beta'}.
\end{align*}

We set $\beta'=\alpha\beta/4n$. Then, substituting $\beta_{\mathit{JL}}=\frac{\alpha\beta^2}{64n}$, we get that with probability at least $1-\beta/4$,
\begin{equation}\label{eq:prob_G_A}
\Pr\limits_{(\mathbf{x},y)\sim D}[(\mathbf{x},y)\in \mathcal{G}_A] \geq 1-\beta'.
\end{equation}

Therefore, with probability $1-\beta/4$, the sampled matrix $A$ satisfies inequality~\eqref{eq:prob_G_A}, that is, a point $(\mathbf{x},y)\sim D$ is in $\mathcal{G}_A$ with probability at least $1-\beta'$.
Furthermore, by union bound, $\forall (\mathbf{x},y)\in S$ it holds that $(\mathbf{x},y) \in \mathcal{G}_A$, with probability at least $1-n\beta'\geq 1-\beta/4$.

For the remainder of the proof, we condition on the event that:
\begin{enumerate}
\item $\Pr\limits_{(\mathbf{x},y)\sim D}[(\mathbf{x},y)\in \mathcal{G}_A] \geq 1-\beta'$ holds for $A$ and 
\item $S \subseteq \mathcal{G}_A$, that is, $\mathbf{w}_A^*$ has margin at least $96\gamma/100$ on $S_A$.
\end{enumerate}
This event occurs with probability at least $1-\beta/4-\beta/4=1-\beta/2$.

\begin{claim}\label{cl:emp_error_approx}
If $n=\frac{1}{\alpha\eps\gamma^2} \cdot \mathrm{polylog}\left(\frac{1}{\alpha}, \frac{1}{\beta}, \frac{1}{\delta}, \frac{1}{\eps}, \frac{1}{\gamma}\right)$, then for the hypothesis $\mathbf{\hat{w}}$ returned by $\mathcal{F}$, with probability $1-\beta/4$, it holds that
\begin{equation}\label{eq:emp_risk_approx_a}
\frac{1}{n}\sum_{(\mathbf{x}_A,y)\in S_A}\mathbbm{1}\left\{y\cdot \langle \mathbf{\hat{w}}, \mathbf{x}_A\rangle < \frac{\gamma}{10}\right\} \leq \frac{\alpha}{4}.
\end{equation}
\end{claim}

\begin{proof}[Proof of Claim~\ref{cl:emp_error_approx}]
Since $\mathbf{w}^*_A$ has margin at least $96\gamma/100$ for all points in $S_A$, it holds that $\min\limits_{\mathbf{w}\in\mathcal{C}} \mathcal{L}(\mathbf{w}; S_A) \leq \mathcal{L}(\mathbf{w}^*_A; S_A) =0$. 
Substituting $\|\mathcal{C}\|_2=2$, $L=100/86\gamma$, and $m=O\left(\frac{\log(n/\alpha\beta)}{\gamma^2}\right)$ into~\eqref{eq:BST}, dividing by $n$, and simplifying the expression, we get that with probability at least $1-\beta/4$,
\begin{equation}\label{eq:emp_risk_approx}
\frac{1}{n}\mathcal{L}(\mathbf{\hat{w}}; S_A) = \frac{1}{n\eps\gamma^2} \cdot \mathrm{polylog}\left(n, \frac{1}{\alpha}, \frac{1}{\beta}, \frac{1}{\delta}\right).
\end{equation}
It also holds that:
\begin{align*}
\frac{1}{n}\mathcal{L}(\mathbf{\hat{w}}; S_A) 
& = \frac{1}{n}\sum_{(\mathbf{x}_A,y)\in S_A} \mathbbm{1}\left\{y\cdot \langle \mathbf{\hat{w}}, \mathbf{x}_A\rangle < \frac{96\gamma}{100}\right\}\cdot \left( \frac{96}{86}-\frac{y\cdot \langle \mathbf{\hat{w}}, \mathbf{x}_A\rangle}{86\gamma/100}\right)\\
& \geq \frac{1}{n}\sum_{(\mathbf{x}_A,y)\in S_A}\mathbbm{1}\left\{y\cdot \langle \mathbf{\hat{w}}, \mathbf{x}_A\rangle < \frac{\gamma}{10}\right\}
\end{align*}
By the latter and inequality~\eqref{eq:emp_risk_approx}, it follows that with probability at least $1-\beta/4$,
\[\frac{1}{n}\sum_{(\mathbf{x}_A,y)\in S_A}\mathbbm{1}\left\{y\cdot \langle \mathbf{\hat{w}}, \mathbf{x}_A\rangle < \frac{\gamma}{10}\right\} = \frac{1}{n\eps\gamma^2} \cdot \mathrm{polylog}\left(n, \frac{1}{\alpha}, \frac{1}{\beta}, \frac{1}{\delta}\right).\]
Therefore, for $n=\frac{1}{\alpha\eps\gamma^2} \cdot \mathrm{polylog}\left(\frac{1}{\alpha}, \frac{1}{\beta}, \frac{1}{\delta}, \frac{1}{\eps}, \frac{1}{\gamma}\right)$ with probability at least $1-\beta/4$, 
\begin{equation*}
\frac{1}{n}\sum_{(\mathbf{x}_A,y)\in S_A}\mathbbm{1}\left\{y\cdot \langle \mathbf{\hat{w}}, \mathbf{x}_A\rangle < \frac{\gamma}{10}\right\} \leq \frac{\alpha}{4}.
\end{equation*}
This completes the proof of the claim.
\end{proof}

\begin{claim}\label{cl:generalization_approx}
If $n=\frac{1}{\alpha\eps\gamma^2}\cdot \mathrm{polylog}\left(\frac{1}{\alpha}, \frac{1}{\beta}, \frac{1}{\delta}, \frac{1}{\eps}, \frac{1}{\gamma}\right)$, then with probability $1-\beta/2$, the error of the returned classifier $\mathbf{\hat{w}}^\top A$ on distribution $D$ is 
\begin{equation}\label{eq:generalization_approx}
\Pr\limits_{(\mathbf{x},y)\sim D}[y\cdot \langle \mathbf{\hat{w}}^\top A,\mathbf{x}\rangle < 0] \leq \alpha.
\end{equation}
\end{claim}

\begin{proof}[Proof of Claim~\ref{cl:generalization_approx}]
Let $D_A$ denote the probability distribution with domain $\mathcal{B}_2^{m} \times \{\pm1\}$, from which a sample $(\mathbf{x}_A, y)\in S_A$ is drawn. Let us also denote by $D_{|\mathcal{G}_A}$ distribution $D$ restricted on $\mathcal{G}_A$. In our conditioned probability space, $S_A \sim D_A^n$, where the probability density function of $D_A$ would be defined as \[\Pr\limits_{(\mathbf{x}_A,y)\sim D_A}[\mathbf{x}_A = \mathbf{x}' \wedge y=y'] = \Pr\limits_{(\mathbf{x},y)\sim D_{|\mathcal{G}_A}} \left[\frac{A\mathbf{x}}{\|A\mathbf{x}\|_2}=\mathbf{x}' \wedge y= y'\right].\]
Let $\mathcal{H}=\left\{ h:\{\mathbf{x}_A \mid (\mathbf{x},y)\in \mathcal{G}_A\} \rightarrow \{\pm1\} \text{ s.t. } h(\mathbf{x})=\mathrm{sign}(\langle \mathbf{w}, \mathbf{x}\rangle) \text{ for some } \mathbf{w}\in\mathcal{B}_2^m\right\}$ be a concept class of threshold functions in $\mathcal{B}_2^m$. We will use the following generalization bound.

\begin{lemma}[\cite{AB09}]\label{th:AB09}
Let $\mathcal{H}$ be a set of $\{\pm1\}$-valued functions defined on a set $X$ and $P$ is a probability distribution on $Z=X\times\{\pm1\}$.
For $\eta\in(0,1)$, $\zeta>0$, and $n\in\mathbb{N}^+$, $\Pr_{z\sim P^n}\left[\exists h\in \mathcal{H}: \mathrm{err}_{P}(h) > (1+\zeta) \hat{\mathrm{err}}_{z}(h)+ \eta\right] \le 4\Pi_{\mathcal{H}}(2n)\exp\left(-\frac{\eta\zeta n}{4(\zeta+1)}\right),$
where $\hat{\mathrm{err}}_{z}(h)$ is the empirical error of $h$ on the sample set $z$ and $\Pi_{\mathcal{H}}(\cdot)$ the growth function of $\mathcal{H}$. \end{lemma}

Setting $\eta=\alpha/4$ and $\zeta=1$, we get that:
\begin{equation*}
\Pr_{S_A\sim D_A^n} \left[\exists h\in \mathcal{H}: \mathrm{err}_{D_A}(h) > 2\cdot \frac{1}{n}\sum_{(\mathbf{x}_A,y)\in S_A}\mathbbm{1}\left\{h(\mathbf{x}_A) \neq y\right\} + \frac{\alpha}{4}\right] \leq 4\Pi_{\mathcal{H}}(2n)\exp\left(-\frac{\alpha n}{32}\right).
\end{equation*}
By Theorems 3.4 and 3.7 of~\cite{AB09}, we get that $\mathrm{VCdim}(\mathcal{H})=m+1$ and $\Pi_{\mathcal{H}}(2n) \leq (2n)^{m+1}+1$.
Substituting $m=O\left(\log(1/\alpha\beta)/\gamma^2\right)$ and for $n=\frac{1}{\alpha\gamma^2} \cdot \mathrm{polylog}\left(\frac{1}{\alpha}, \frac{1}{\beta}, \frac{1}{\gamma}\right)$, $4\Pi_{\mathcal{H}}(2n)\exp(-\alpha n/32) \leq \beta/4$. Therefore, with probability at least $1-\beta/4$,
\begin{equation}\label{eq:general_approx}
\mathrm{err}_{D_A}(f_{\mathbf{\hat{w}}}) \leq 2\cdot \frac{1}{n}\sum_{(\mathbf{x}_A,y)\in S_A}\mathbbm{1}\left\{y\cdot\langle \mathbf{\hat{w}}, \mathbf{x}_A \rangle<0\right\} +\frac{\alpha}{4}.
\end{equation}

By Claim~\ref{cl:emp_error_approx}, $\frac{1}{n}\sum\limits_{(\mathbf{x}_A,y)\in S_A}\mathbbm{1}\left\{y\cdot\langle \mathbf{\hat{w}}, \mathbf{x}_A \rangle<0\right\} \leq \frac{1}{n}\sum\limits_{(\mathbf{x}_A,y)\in S_A}\mathbbm{1}\left\{y\cdot\langle \mathbf{\hat{w}}, \mathbf{x}_A \rangle<\frac{\gamma}{10}\right\} \leq \frac{\alpha}{4}$ holds with probability $1-\beta/4$, if $n=\frac{1}{\alpha\eps\gamma^2}\cdot \mathrm{polylog}\left(\frac{1}{\alpha}, \frac{1}{\beta}, \frac{1}{\delta}, \frac{1}{\eps}, \frac{1}{\gamma}\right)$. 

Therefore, by inequality~(\ref{eq:general_approx}), if $n=\frac{1}{\alpha\eps\gamma^2}\cdot \mathrm{polylog}\left(\frac{1}{\alpha}, \frac{1}{\beta}, \frac{1}{\delta}, \frac{1}{\eps}, \frac{1}{\gamma}\right)$, then with probability at least $1-\beta/2$,
\[\mathrm{err}_{D_A}(f_{\mathbf{\hat{w}}}) \leq 2\cdot\frac{\alpha}{4}+\frac{\alpha}{4}=\frac{3\alpha}{4}.\]

Equivalently, with probability at least $1-\beta/2$, 
\[ \Pr\limits_{(\mathbf{x},y)\sim D_{|\mathcal{G}_A}}[y\cdot \langle \mathbf{\hat{w}}^\top A,\mathbf{x}\rangle < 0] = \Pr\limits_{(\mathbf{x}_A,y)\sim D_{A}}[y\cdot \langle \mathbf{\hat{w}},\mathbf{x}_A\rangle < 0] = \mathrm{err}_{D_A}(f_{\mathbf{\hat{w}}}) \leq \frac{3\alpha}{4}.\]
Since, by Condition 1., $\Pr\limits_{(\mathbf{x},y)\sim D}[(\mathbf{x},y)\notin \mathcal{G}_A] \leq \beta' \leq \frac{\alpha}{4}$, it follows that with probability $1-\beta/2$,
\begin{align*}
\Pr\limits_{(\mathbf{x},y)\sim D}[y\cdot \langle \mathbf{\hat{w}}^\top A,\mathbf{x}\rangle < 0] 
& \leq \Pr\limits_{(\mathbf{x},y)\sim D_{|\mathcal{G}_A}}[y\cdot \langle \mathbf{\hat{w}}^\top A,\mathbf{x}\rangle < 0] \cdot (1-\beta')+  1\cdot \beta'\\
&\leq \frac{3\alpha}{4} \cdot (1-\beta') + \beta' \\
& \leq \frac{3\alpha}{4}+\frac{\alpha}{4} \leq \alpha.
\end{align*}
This completes the proof of the claim.
\end{proof}
Accounting for the probability that we are not in the conditioned space, we conclude that if $n=\frac{1}{\alpha\eps\gamma^2}\cdot \mathrm{polylog}\left(\frac{1}{\alpha}, \frac{1}{\beta}, \frac{1}{\delta}, \frac{1}{\eps}, \frac{1}{\gamma}\right)$, then with probability $1-\beta/2-\beta/2=1-\beta$, $\mathrm{err}_{D}(f_{\mathbf{\hat{w}}^\top A}) \leq  \alpha$. This completes the proof of the theorem.
\end{proof}

\begin{theorem}[Privacy guarantee]\label{th:priv_approximate}
Algorithm $\mathcal{A}_{\alpha,\beta,\eps,\delta,\gamma}$ is $(\eps,\delta)$-differentially private.
\end{theorem}
\begin{proof}[Proof of Theorem \ref{th:priv_approximate}]
By Lemma~\ref{lem:post-processing}, $(\eps,\delta)$-differential privacy is closed under post-processing~\cite{DMNS06}. So it suffices to show that an algorithm $\mathcal{N}$ that is the same as $\mathcal{A}_{\alpha,\beta,\eps,\delta,\gamma}$ except that it returns $\mathbf{\hat{w}}$ instead of $\mathbf{\hat{w}}^\top A$, is $(\eps,\delta)$-DP.

Let data sets $S\sim S'$ such that $S=S'\setminus\{(\mathbf{x}', y')\}\cup\{(\mathbf{x},y)\}$. Let $U=\left\{\pm1/\sqrt{m}\right\}^{m\times d}$. If we fix a matrix $A\in U$, then $S$ and $S'$ correspond to $\mathcal{F}$'s inputs $S_A$ and $S'_A = S_A \setminus \{(\mathbf{x}'_A, y')\}\cup\{(\mathbf{x}_A,y)\}$, respectively. Recall from Theorem~\ref{th:BST}, that $\mathcal{F}$ is $(\eps,\delta)$-DP.
For any measurable set $\mathcal{R}\subseteq\mathbb{R}^m$, 
\begin{align*}
\Pr[\mathcal{N}(S) \in \mathcal{R}] & = \sum_{A\in U} \Pr[A]\cdot\Pr[\mathcal{F}(S_A) \in \mathcal{R} \mid A] \\
& \leq \sum_{A\in U} \Pr[A]\cdot\left(\exp(\eps)\Pr[\mathcal{F}(S_A')\in \mathcal{R} \mid A] + \delta\right) \tag{by Theorem~\ref{th:BST}} \\ 
& = \exp(\eps)\sum_{A\in U} \Pr[A]\cdot\Pr[\mathcal{F}(S_A')\in \mathcal{R} \mid A]  + \delta \sum_{A\in U} \Pr[A]\\
& = \exp(\eps) \Pr[\mathcal{N}(S') \in \mathcal{R}] + \delta.
\end{align*}
 Therefore, $\mathcal{N}$ is $(\eps,\delta)$-DP, and so is  $\mathcal{A}_{\alpha,\beta,\eps,\delta,\gamma}$.
\end{proof}

\section{A Pure Differentially Private Algorithm}\label{sec:pure}
As previously, algorithm $\mathcal{A}_{\alpha,\beta,\eps,\gamma}$ takes as input a sample set $S\sim D^n$ and performs dimension reduction. In this reduced space, it defines a net of hypotheses and uses the \emph{Exponential Mechanism}~\citep{MT07} to choose a good hypothesis with respect to the sample set.

The Exponential Mechanism is a well-known algorithm, which serves as a building block for many differentially private algorithms. The mechanism is used in cases where we need to choose the optimal output with respect to some utility function on the data set. More formally, let $\mathcal{O}$ denote the range of the outputs and let $u:\mathcal{S}^n\times \mathcal{O} \rightarrow \mathbb{R}$ be the utility function which maps the data set - output pairs to utility scores. 

An important notion in differential privacy is that of the \textit{sensitivity} of a function. Intuitively, it represents the maximum change that the change of a single data point can incur on the output of the function, and as a result, it drives the amount of uncertainty we need in order to ensure privacy. The sensitivity of the utility function, which is only with respect to the data set, is defined as
\[\Delta u = \max\limits_{o\in\mathcal{O}} \max\limits_{\substack{S,S'\in\mathcal{S}^n\\ S\sim S'}} |u(S,o)-u(S',o)|.\]
 
\begin{definition}[Exponential Mechanism,~\cite{MT07}]
Let data set $S\in \mathcal{S}^n$, range $\mathcal{O}$, and utility function $u:\mathcal{S}^n\times \mathcal{O} \rightarrow \mathbb{R}$. The {\em Exponential Mechanism} $\mathcal{M}_E(S,u,\mathcal{O})$ selects and outputs an element $o\in\mathcal{O}$ with probability proportional to $\exp\left(\frac{\eps \cdot u(S,o)}{2\Delta u}\right)$.
\end{definition}

The Exponential Mechanism has the following guarantees.

\begin{lemma}[Privacy and Accuracy of the Exponential Mechanism,~\cite{MT07}]\label{lem:expmech}
The Exponential Mechanism is $(\eps,0)$-differentially private and with probability at least $1-\delta$
\[| \max\limits_{o\in\mathcal{O}} u(S,o) - u(\mathcal{M}_E(S,u,\mathcal{O} ))| \leq \frac{2\Delta u}{\eps}\ln\left(\frac{|\mathcal{O}|}{\delta}\right).\]
\end{lemma}
\begin{algorithm}[H]
\caption{$\mathcal{A}_{\alpha,\beta,\eps,\gamma}(S)$}\label{pure}
\begin{algorithmic}[1]
\State Choose a random matrix $A\in \mathbb{R}^{m \times d}$, where $m=O\left(\frac{\log(1/\beta_{\mathit{JL}})}{\gamma^2}\right)$, $\beta_{\mathit{JL}}=\alpha\beta^2/64n$, and $$A_{ij} = \begin{cases}
        +1/\sqrt{m} &\text{w.p.~} 1/2\\
        -1/\sqrt{m} &\text{w.p.~} 1/2.
        \end{cases}$$\;
\State Define $S_A \gets \left\{\left(A\mathbf{x} / \|A\mathbf{x} \|_2,y\right) \mid (\mathbf{x},y)\in S\right\}$.\;
\State Let $\mathcal{W} $ be a $\frac{\gamma}{10}-$net of $\mathcal{B}^m_2$.\;
\State Define the utility function $u: (\mathcal{B}_2^m\times\{\pm1\} )^n \times \mathcal{W} \rightarrow [-1,0]$: $$ u(D, \mathbf{w}) =-\frac{1}{n}\cdot\sum\limits_{(\mathbf{x},y)\in D}\mathbbm{1}\left\{y\cdot \langle \mathbf{w}, \mathbf{x}\rangle < \frac{\gamma}{10}\right\}.$$\;
\State $\mathbf{\hat{w}} \gets \mathcal{M}_E(S_A, u, \mathcal{W})$\;
\State \Return $\mathbf{\hat{w}}^\top A$.\;
\end{algorithmic}
\end{algorithm}

In the following we provide the sample complexity and privacy guarantees of our algorithm. 

\begin{theorem}[Sample Complexity]\label{th:UB_pure}
Algorithm $\mathcal{A}_{\alpha,\beta,\eps,\gamma}$ is an $(\alpha, \beta, \gamma)$-learner with sample complexity
\[n=\frac{1}{\alpha\eps\gamma^2}\cdot \mathrm{polylog}\left(\frac{1}{\alpha}, \frac{1}{\beta}, \frac{1}{\eps}, \frac{1}{\gamma}\right).\]
\end{theorem}

The proof of Theorem~\ref{th:UB_pure} is very similar to that of the previous section, except for the bound on the empirical loss of the learned classifier. We prove this part here, and the full proof can be found in Appendix~\ref{app:pure}.

\begin{claim}\label{cl:emp_error_pure1}
If $n=\frac{1}{\alpha\eps\gamma^2}\cdot \mathrm{polylog}\left(\frac{1}{\alpha}, \frac{1}{\beta},  \frac{1}{\eps}, \frac{1}{\gamma}\right)$, then with probability $1-\beta/4$, for hypothesis $\mathbf{\hat{w}}$ returned by the Exponential Mechanism it holds that
\begin{equation}\label{eq:emp_risk_pure_a1}
\frac{1}{n}\sum_{(\mathbf{x}_A,y)\in S_A}\mathbbm{1}\left\{y\cdot \langle \mathbf{\hat{w}}, \mathbf{x}_A\rangle < \frac{\gamma}{10}\right\} \leq \frac{\alpha}{4}.
\end{equation}
\end{claim}

\begin{proof}[Proof of Claim~\ref{cl:emp_error_pure1}]
Every point in $\mathcal{B}_2^m$ is within $\gamma/10$ from a center of $\mathcal{W}$.
Let $\mathbf{w}_c^*$ be the center within $\gamma/10$ from $\mathbf{w}^*_A$, that is, 
\begin{equation}\label{eq:w_center1}
\|\mathbf{w}_A^* - \mathbf{w}_c^*\|_2 \leq \gamma/10.
\end{equation}

Recall that in our conditioned probability space, 
\begin{equation}\label{eq:w_star_A1}
y\cdot\langle \mathbf{w}^*_A, \mathbf{x}_A \rangle \geq 96\gamma/100
\end{equation}
holds for all $(\mathbf{x}_A,y)\in S_A$.
Therefore, for all $(\mathbf{x}_A,y)\in S_A$, 
\begin{align*}
y\cdot\langle \mathbf{w}^*_c, \mathbf{x}_A \rangle 
&= y\cdot\langle \mathbf{w}^*_A, \mathbf{x}_A \rangle - y\cdot\langle \mathbf{w}^*_A - \mathbf{w}^*_c, \mathbf{x}_A \rangle \\
&\geq y\cdot\langle \mathbf{w}^*_A, \mathbf{x}_A \rangle - \|\mathbf{w}_A^* - \mathbf{w}_c^*\|_2\cdot \|\mathbf{x}_A\|_2 \\
&\geq 96\gamma/100 - \gamma/10 = 86\gamma/10 > \gamma/10. \tag{by inequalities~\eqref{eq:w_center1}, \eqref{eq:w_star_A1}}
\end{align*}

It follows that 
\begin{equation}\label{eq:opt_exp1}
\max\limits_{\mathbf{w}\in\mathcal{W}} u(S_A,\mathbf{w}) \geq u(S_A, \mathbf{w}^*_c) = -\frac{1}{n}\sum_{(\mathbf{x}_A,y)\in S_A}\mathbbm{1}\left\{y\cdot \langle \mathbf{w}, \mathbf{x}_A\rangle < \frac{\gamma}{10}\right\} = 0.
\end{equation}

By Lemma~\ref{lem:expmech} and inequality~\eqref{eq:opt_exp1}, with probability at least $1-\beta/4$, it holds that:
\begin{equation}\label{eq:emp_risk_pure1}
\frac{1}{n}\sum_{(\mathbf{x}_A,y)\in S_A}\mathbbm{1}\left\{y\cdot \langle \mathbf{\hat{w}}, \mathbf{x}_A\rangle < \frac{\gamma}{10}\right\} \leq \frac{2}{n\eps}(\ln(|\mathcal{W}|) + \ln(4/\beta))
\end{equation}

It is a well-known result that the covering number of an $m$-dimensional unit ball by balls of radius $\gamma/10$ is at most $O\left(\left(\frac{1}{\gamma/10}\right)^m\right)$. Therefore, substituting $m=O\left(\frac{\log(n/\alpha\beta)}{\gamma^2}\right)$, it follows that 
\begin{equation*}\label{eq:coverno1}
\ln|\mathcal{W}| = \frac{1}{\gamma^2}\cdot \mathrm{polylog}\left(n,\frac{1}{\alpha}, \frac{1}{\beta}, \frac{1}{\gamma}\right).
\end{equation*} 

Thus, by inequality~\eqref{eq:emp_risk_pure1}, if $n=\frac{1}{\alpha\eps\gamma^2}\cdot\mathrm{polylog}\left(\frac{1}{\alpha}, \frac{1}{\beta}, \frac{1}{\gamma}, \frac{1}{\eps}\right)$ then with probability at least $1-\beta/4$,
\begin{equation*}
\frac{1}{n}\sum_{(\mathbf{x}_A,y)\in S_A}\mathbbm{1}\left\{y\cdot \langle \mathbf{\hat{w}}, \mathbf{x}_A\rangle < \frac{\gamma}{10}\right\} \leq \frac{\alpha}{4}.
\end{equation*}
This completes the proof of the claim. \end{proof}

\begin{theorem}[Privacy guarantee]\label{th:priv_pure}
Algorithm $\mathcal{A}_{\alpha,\beta,\eps,\gamma}$ is $(\eps,0)$-differentially private.
\end{theorem}

\begin{proof}[Proof of Theorem \ref{th:priv_pure}]
The sensitivity of the utility function is \[\Delta u = \max\limits_{\mathbf{w}\in\mathcal{W}} \max\limits_{\substack{Z,Z'\in\left(\mathcal{X}\times\{\pm1\}\right)^n \\ Z\sim Z'}} {| u(Z,\mathbf{w})-u(Z', \mathbf{w})|} \leq \frac{1}{n}.\] It follows by Lemma~\ref{lem:expmech} that $\mathcal{M}_E$ is $(\eps,0)$-DP.

By Lemma~\ref{lem:post-processing}, $(\eps,0)$-differential privacy is closed under post-processing. Therefore it suffices to show that an algorithm $\mathcal{N}$ that is the same as $\mathcal{A}_{\alpha,\beta,\eps,\gamma}$ except that it returns $\mathbf{\hat{w}}$ instead of $\mathbf{\hat{w}}^\top A$, is $(\eps,0)$-DP.

Let $S$ and $S'$ be two neighboring sample sets such that $S=S'\setminus\{(\mathbf{x}', y')\}\cup\{(\mathbf{x},y)\}$. Let $U=\left\{\pm\frac{1}{\sqrt{m}}\right\}^{m\times d}$. If we fix a matrix $A\in U$, the sample sets $S$ and $S'$ would correspond to $\mathcal{M}_E$'s inputs $S_A$ and $S'_A = S_A \setminus \{(\mathbf{x}'_A, y')\}\cup\{(\mathbf{x}_A,y)\}$, respectively. For any measurable set $\mathcal{R}\subseteq\mathbb{R}^m$, it holds that
\begin{align*}
\Pr[\mathcal{N}(S) \in \mathcal{R}] & = \sum_{A\in U} \Pr[A]\cdot\Pr[\mathcal{M}_E(S_A) \in \mathcal{R} \mid A] \\
& \leq \sum_{A\in U} \Pr[A]\cdot\exp(\eps)\Pr[\mathcal{M}_E(S_A')\in \mathcal{R} \mid A] \tag{since $\mathcal{M}_E$ is $(\eps,0)$-DP} \\ 
& = \exp(\eps)\sum_{A\in U} \Pr[A]\cdot\Pr[\mathcal{M}_E(S_A')\in \mathcal{R} \mid A]\\
& = \exp(\eps) \Pr[\mathcal{N}(S') \in \mathcal{R}].
\end{align*}
 Therefore, $\mathcal{N}$ is $(\eps,0)$-DP, and so is  $\mathcal{A}_{\alpha,\beta,\eps,\gamma}$.
\end{proof}

We note that the running time of algorithm $\mathcal{A}_{\alpha,\beta,\eps,\gamma}$ is exponential in $1/\gamma$, due to the use of the exponential mechanism to select from a net of size $\exp(1/\gamma)$.  Although there are efficient implementations of the exponential mechanism for optimizing over \emph{continuous} domains \citep{BST14}, relaxing the domain to be continuous would lead to higher sample complexity.


\section{A Sample Complexity Lower Bound for Pure Differential Privacy}\label{sec:lb}

In this section we prove a lower bound on the sample complexity of any $(\eps,0)$-differentially private algorithm for learning a large-margin halfspace.

\begin{theorem}\label{th:LB}
Any $(\eps,0)$-differentially private $(\frac{1}{10},\frac{1}{10} ,\gamma)$-learner for halfspaces in $\R^{\Omega(1/\gamma^2)}$ requires $\Omega\left(1 / \eps \gamma^2\right)$ samples.
\end{theorem}
\begin{proof}
Our proof uses a standard \emph{packing argument}.  We construct distributions $D^{(1)},\dots,D^{(K)}$ over $\mathcal{B}^d_2\times \{\pm 1\}$ for $d = 1/1000\gamma^2$ and $K = 2^{d/20}$.  We will construct these distributions so that no classifier is simultaneously accurate for two distinct distributions $D^{(i)}$ and $D^{(j)}$.  This will imply that $n = \Omega(\log(K)/\eps) = \Omega(1/\eps \gamma^2)$ samples are necessary to achieve $(\eps,0)$-differential privacy.

Each distribution $D^{(i)}$ is defined with respect to a halfspace $\mathbf{w}^{(i)}\in\{\pm1 / \sqrt{d}\}^d$, and has margin $\gamma$ with respect to this halfspace.  That is
$$\Pr_{(\mathbf{x},y)\sim D^{(i)}}[y\cdot \langle \mathbf{w}^{(i)},\mathbf{x}\rangle \geq \gamma]=1.$$ In addition, $\mathbf{x}$ is distributed uniformly at random on the remaining surface of $\mathcal{B}^d_2$ so that is does not violate the margin and $y = \mathrm{sign}(\langle \mathbf{w}^{(i)},\mathbf{x}\rangle)$. Formally, if $U$ denotes the uniform distribution on $\mathcal{B}^d_2$ and $f_U$ is its density function, then the probability density function of $X$ where $(X,y)\sim D^{(i)}$, is \[f_{X}(\mathbf{x'}) = \begin{cases} f_{U}(\mathbf{x'})/\Pr_{\mathbf{x}\sim U}[|\langle \mathbf{w}^{(i)},\mathbf{x}\rangle|\geq \gamma], &\mbox{if $| \langle \mathbf{w}^{(i)},\mathbf{x'}\rangle|\geq \gamma$} \\ 0, & \mbox {otherwise}. \end{cases}\]

Using standard constructions of error correcting codes, there exists a set $\mathbf{w}^{(1)}, \ldots, \mathbf{w}^{(K)}$ such that the Hamming distance of any pair $i \neq j$, is $\mathrm{Ham}(\mathbf{w}^{(i)}, \mathbf{w}^{(j)})\geq d/10$. This implies that $\langle \mathbf{w}^{(i)}, \mathbf{w}^{(j)} \rangle \leq 4/5$.  

Let $\{\mathbf{w}^{(1)}, \ldots, \mathbf{w}^{(K)}\}$ be such a set and let $D^{(1)},\dots,D^{(K)}$ be the resulting distributions.  The crux of the proof is in establishing the following claim about this set of distributions.  For each distribution $D^{(i)}$, we define the set
$$
\mathcal{G}^{(i)} = \left\{ \mathbf{\hat{w}} \in \mathcal{B}^d_{2} : \mathrm{err}_{D^{(i)}}(\mathbf{\hat{w}}) \leq \frac{1}{10} \right\}
$$
of all classifiers that have error at most $1/10$ on the distribution $D^{(i)}$.
\begin{claim}\label{cl:disjoint}
For every $i \neq j$, $\mathcal{G}^{(i)}$ and $\mathcal{G}^{(j)}$ are disjoint.
\end{claim}

Using this claim, we can complete the proof as follows.  Let $S^{(1)} \sim \left(D^{(i)}\right)^n$ denote a random iid sample of $n$ examples from $D^{(i)}$. Let $\mathcal{A}$ be an $(\eps,0)$-differentially private $(\frac{1}{10},\frac{1}{10},\gamma)$ learner.  By privacy and accuracy, we have for every $i \in \{2,3,\dots,K\}$
\[\Pr[\mathcal{A}(S^{(1)})\in \mathcal{G}^{(i)}] \geq \exp(-n\eps)\Pr[\mathcal{A}(S^{(1)})\in \mathcal{G}^{(1)}] \geq \frac{9}{10} \exp(-n\eps).\]
Since the sets $\mathcal{G}^{(i)}$ are disjoint, 
$$
\Pr[\mathcal{A}(S^{(1)})\notin \mathcal{G}^{(1)}] \geq \sum_{i=2}^{K} \Pr[\mathcal{A}(S^{(1)})\in \mathcal{G}^{(i)}] \geq \sum_{i=2}^{K} \frac{9}{10} \exp(-n\eps) = \frac{9}{10} (K-1) \exp(-n\eps).
$$
Since, by accuracy, $\Pr[\mathcal{A}(S^{(1)})\notin \mathcal{G}^{(1)}]\leq \frac{1}{10}$, it follows that
\[\frac{9}{10} (K-1)\exp(-n\eps) \leq \frac{1}{10} \]
Rearranging, and substituting our choice of $K$, we conclude $n = \Omega(1/\eps \gamma^2)$.

Let us now prove Claim~\ref{cl:disjoint}, which will complete the proof.
\begin{proof}[Proof of Claim~\ref{cl:disjoint}]
We will show that for an arbitrary $\mathbf{\hat{w}}\in\mathcal{B}^d_2$, and every $i \neq j$, if $\mathrm{err}_{D^{(i)}}(\mathbf{\hat{w}}) \leq 1/10$ then $\mathrm{err}_{D^{(j)}}(\mathbf{\hat{w}}) > 1/10$.  Let $U_i$ be the distribution over $\mathcal{B}^d_2\times\{\pm 1\}$ such that $(\mathbf{x},y)\sim U_i$ if $\mathbf{x}\sim U$ is uniform over the unit sphere in $\R^d$ and $y=\mathrm{sign}(\langle \mathbf{w}^{(i)}, \mathbf{x}\rangle)$.
Define the probability
\[p_\gamma=\Pr\limits_{\mathbf{x}\sim U}[|\langle \mathbf{w}, \mathbf{x}\rangle | < \gamma] \]
of a uniformly distributed point on $\mathcal{B}^d_2$ lying within margin $\gamma$ of a unit vector $\mathbf{w}$.
Probability $p_\gamma$ remains unchanged if we replace $\mathbf{w}$ with any other unit vector and, more specifically,
\[p_\gamma = \Pr\limits_{(\mathbf{x}, y)\sim U_i}[|\langle \mathbf{w}^{(i)}, \mathbf{x}\rangle | < \gamma]\] holds for all $U_i, i\in[K]$. 

The next lemma will allow us to show that $U_i$ is not too far from $D^{(i)}$.
\begin{lemma}\label{margin_mass}
If $d=\frac{1}{1000\gamma^2}$ then $p_\gamma \leq 0.2.$
\end{lemma}
\begin{proof}
Consider the following sampling process from the uniform distribution on the sphere: choose each coordinate $x_i\sim \mathcal{N}(0,1)$ and normalize with $\|\mathbf{x}\|_2$. By the symmetric property of multi-dimensional gaussian vectors, we know that the projection of $\mathbf{x}$ on a unit vector $\mathbf{w}$ is distributed as $x_1/\|\mathbf{x}\|_2$, where $x_1 \sim \mathcal{N}(0,1)$ is the first coordinate of $\mathbf{x}$ and $\|\mathbf{x}\|_2^2 \sim \tilde{\chi}^2(d)$ is the square of the normalization factor. The probability of a point having margin more than $\gamma$ from $\mathbf{w}$ is:\begin{align*}
1-p_\gamma & = \Pr\limits_{\mathbf{x} \sim U}\left[\frac{|x_1|}{\|\mathbf{x}\|_2}   \geq \gamma\right]\\
&\geq \Pr\limits_{\mathbf{x}\sim U}\left[\left(|x_1| \geq \frac{1}{10} \right)\land \left(\|\mathbf{x}\|_2 \leq \frac{1}{10\gamma}\right)\right]\\
& = 1 - \Pr\limits_{\mathbf{x}\sim U}\left[\left(|x_1| < \frac{1}{10}\right) \lor \left(\|\mathbf{x}\|_2 > \sqrt{10d}\right)\right] \tag{$d=\frac{1}{1000\gamma^2}$}\\
& \geq 1-  \Pr\limits_{x_1\sim \mathcal{N}(0,1)}\left[|x_1| < \frac{1}{10}\right] -  \Pr\limits_{\|\mathbf{x}\|^2_2\sim  \tilde{\chi}^2(d)}\left[\|\mathbf{x}\|_2^2 > 10d\right]
\end{align*}

By the tables of the standard normal distribution we have that $\Pr\limits_{x_1\sim \mathcal{N}(0,1)}\left[|x_1| < \frac{1}{10}\right] \leq 0.08$. Also, the mean of a $\tilde{\chi}^2(d)$ distributed variable is $d$. By Markov's inequality, it follows that \[\Pr\limits_{\|\mathbf{x}\|^2_2\sim  \tilde{\chi}^2(d)}\left[\|\mathbf{x}\|_2^2 >10d\right] \leq d/10d=1/10.\]

Thus, $p_\gamma \leq 0.18$.
\end{proof}

We can apply the preceding lemma to relate $\mathrm{err}_{U_i}(\mathbf{\hat{w}})$ to $\mathrm{err}_{D^{(i)}}(\mathbf{\hat{w}})$.  Specifically, for any $\mathbf{\hat{w}}$ with $\mathrm{err}_{D^{(i)}}(\mathbf{\hat{w}}) \leq 0.10$, it holds that:
\begin{align*}
\mathrm{err}_{U_i}(\mathbf{\hat{w}}) 
&= \Pr\limits_{(\mathbf{x}, y)\sim U_i}[\mathrm{sign}(\langle \mathbf{\hat{w}}, \mathbf{x}\rangle) \neq \mathrm{sign}(\langle \mathbf{w}^{(i)}, \mathbf{x}\rangle)]\\
& \leq \Pr\limits_{(\mathbf{x}, y)\sim U_i}[\mathrm{sign}(\langle \mathbf{\hat{w}}, \mathbf{x}\rangle) \neq \mathrm{sign}(\langle \mathbf{w}^{(i)}, \mathbf{x}\rangle) \mid |\langle \mathbf{w}^{(i)}, \mathbf{x}\rangle | \geq \gamma] \cdot \Pr\limits_{(\mathbf{x}, y)\sim U_i}[|\langle \mathbf{w}^{(i)}, \mathbf{x}\rangle | \geq \gamma]\\
& \hspace{1.5cm} + \Pr\limits_{(\mathbf{x}, y)\sim U_i}[|\langle \mathbf{w}^{(i)}, \mathbf{x}\rangle | < \gamma]\\
& \leq 0.1+0.2= 0.3
\end{align*}
Therefore,
\begin{equation}\label{one}
\mathrm{err}_{U_i}(\mathbf{\hat{w}})  \leq 0.3.
\end{equation}

Next we will argue that the same vector $\mathbf{\hat{w}}$ cannot have low error with respect to some other distribution $U_j$.  Fix any two vectors $\mathbf{w}^{(i)}, \mathbf{w}^{(j)}$ as in our construction.  Consider the plane defined by these vectors and let $\theta$ be their angle. It holds that $$\Pr_{\mathbf{x}\sim U}[\mathrm{sign}(\langle \mathbf{w}^{(i)}, \mathbf{x}\rangle) = \mathrm{sign}(\langle \mathbf{w}^{(j)}, \mathbf{x}\rangle)] = \frac{2\theta}{2\pi}=\frac{\theta}{\pi} = \frac{\cos^{-1}(\langle \mathbf{w}^{(i)},\mathbf{w}^{(j)}\rangle)}{\pi}.$$ Since $\mathrm{Ham}(\mathbf{w}^{(i)},\mathbf{w}^{(j)})\geq d/10$, $\langle \mathbf{w}^{(i)},\mathbf{w}^{(j)}\rangle \leq \frac{1}{d}(9d/10-d/10) = \frac{8}{10}$. Thus,  \[\Pr_{\mathbf{x}\sim U}[\mathrm{sign}(\langle \mathbf{w}^{(i)}, \mathbf{x}\rangle) = \mathrm{sign}(\langle \mathbf{w}^{(j)}, \mathbf{x}\rangle)] \leq \frac{\cos^{-1}(8/10)}{\pi} =0.21\]

For the error of $\mathbf{\hat{w}}$ on distribution $U_j$ it holds that:
\begin{align*}
&\Pr\limits_{(\mathbf{x}, y)\sim U_j}[\mathrm{sign}(\langle \mathbf{\hat{w}}, \mathbf{x}\rangle) = \mathrm{sign}(\langle \mathbf{w}^{(j)}, \mathbf{x}\rangle)]\\
& \leq \Pr\limits_{(\mathbf{x}, y)\sim U_j}[\mathrm{sign}(\langle \mathbf{\hat{w}}, \mathbf{x}\rangle) \neq \mathrm{sign}(\langle \mathbf{w}^{(i)}, \mathbf{x}\rangle) ] + \Pr\limits_{(\mathbf{x}, y)\sim U_j}[\mathrm{sign}(\langle \mathbf{w}^{(j)}, \mathbf{x}\rangle) = \mathrm{sign}(\langle \mathbf{w}^{(i)}, \mathbf{x}\rangle) ]\\
& \leq 0.3 + 0.21 = 0.51 \tag{by \eqref{one}}
\end{align*}
Therefore, 
\begin{equation}\label{two}
\mathrm{err}_{U_j}(\mathbf{\hat{w}})\geq 0.49.
\end{equation}

Once again, we can relate this to the error on the distribution $D^{(j)}$ as follows.
\begin{align*}
&\Pr\limits_{(\mathbf{x}, y)\sim D^{(j)}}[\mathrm{sign}(\langle \mathbf{\hat{w}}, \mathbf{x}\rangle) \neq \mathrm{sign}(\langle \mathbf{w}^{(j)}, \mathbf{x}\rangle)]\\
& = \Pr\limits_{(\mathbf{x}, y)\sim U_j}[\mathrm{sign}(\langle \mathbf{\hat{w}}, \mathbf{x}\rangle) \neq \mathrm{sign}(\langle \mathbf{w}^{(j)}, \mathbf{x}\rangle)\mid |\langle \mathbf{w}^{(j)}, \mathbf{x}\rangle| \geq \gamma]\cdot  \Pr\limits_{(\mathbf{x}, y)\sim U_j}[|\langle \mathbf{w}^{(j)}, \mathbf{x}\rangle| \geq \gamma]\\
& \geq \Pr\limits_{(\mathbf{x}, y)\sim U_j}[\mathrm{sign}(\langle \mathbf{\hat{w}}, \mathbf{x}\rangle) \neq \mathrm{sign}(\langle \mathbf{w}^{(j)}, \mathbf{x}\rangle)] - \Pr\limits_{(\mathbf{x}, y)\sim U_j}[|\langle \mathbf{w}^{(j)}, \mathbf{x}\rangle | < \gamma] \\
& \geq 0.49-p_\gamma \geq 0.29  \tag{by \eqref{two}}
\end{align*}
Therefore, $\mathrm{err}_{D^{(j)}}(\mathbf{\hat{w}})\geq 0.29$.  Thus, for any $\mathbf{\hat{w}}$, if $\mathrm{err}_{D^{(i)}}(\mathbf{\hat{w}}) \leq 0.1$ then $\mathrm{err}_{D^{(j)}}(\mathbf{\hat{w}}) \geq 0.29$. \end{proof}
This completes the proof of the lower bound.
\end{proof}

\section*{Acknowledgements}
JU was supported by NSF grants CCF-1718088, CCF-1750640, and CNS-1816028. HN and LZ were supported by NSF grants CCF-1750716 and CCF-1909314.

\bibliography{DP_Halfspace}{}
\bibliographystyle{plainnat}

\newpage
\appendix
\section{Algorithm and proofs of Section~\ref{sec:approx}}\label{app:algo}
\subsection{Differentially Private Empirical Risk Minimization Algorithm $\mathcal{F}$}
A complete algorithm, deploying the differentially private stochastic gradient descent algorithm by~\citet{BST14}, is presented below in Algorithm~\ref{algo:complete}.
We denote by $\Pi_{\mathcal{C}}(\cdot)$ the euclidean projection on $\mathcal{C}$ and by $\|\mathcal{C}\|_2$ the diameter of $\mathcal{C}$.
\begin{algorithm}[H]
\caption{$\mathcal{A}_{\alpha,\beta,\eps,\delta,\gamma}(S)$}\label{algo:complete}
\renewcommand{\thealgorithm}{}
\begin{algorithmic}[1]
\State Choose a random matrix $A\in \mathbb{R}^{m \times d}$, where $m=O\left(\frac{\log(1/\beta_{\mathit{JL}})}{\gamma^2}\right)$, $\beta_{\mathit{JL}}=\alpha\beta^2/64n$, and $$A_{ij} = \begin{cases}
        +1/\sqrt{m} &\text{w.p.~} 1/2\\
        -1/\sqrt{m} &\text{w.p.~} 1/2.
        \end{cases}$$\;
\State Define $S_A \gets \left\{\left(A\mathbf{x} / \|A\mathbf{x} \|_2,y\right) \mid (\mathbf{x},y)\in S\right\}$.\;
\State Define the hypothesis set $\mathcal{C} \gets \mathcal{B}_2^m$.\;
\State Define the $\frac{100}{86\gamma}$-Lipschitz loss function $\ell : \mathcal{C}\times\left(\mathcal{B}_2^m\times\{\pm 1\}\right) \to \mathbb{R}$ as $$\ell(\mathbf{w};(\mathbf{x},y)) =\mathbbm{1}{\left\{y\cdot \langle \mathbf{w}, \mathbf{x}\rangle < \frac{96\gamma}{100}\right\}} \cdot \left( \frac{96}{86}-\frac{y\cdot \langle \mathbf{w}, \mathbf{x}\rangle}{86\gamma/100}\right).$$\;
\State Let $\mathbf{\hat{w}} \gets \mathcal{F}(S_A, \ell, (\eps,\delta), \mathcal{C})$.\;
\State \Return $\mathbf{\hat{w}}^\top A$.\;
\vspace{0.5cm}
\Procedure{$\mathcal{F}$}{$D, \ell, (\eps,\delta), \mathcal{C}$}
\For{$i=1$ to $\lceil\log(8/\beta)\rceil$}
\State $\mathbf{\hat{w}}^{(i)} \gets \mathcal{A}_{\textup{Noise-GD}}(D, \ell, (\eps/\lceil\log(8/\beta)\rceil,\delta/\lceil\log(8/\beta)\rceil), \mathcal{C})$
\EndFor
\State $\mathcal{W} \gets \{\mathbf{\hat{w}}^{(1)}, \ldots, \mathbf{\hat{w}}^{(\lceil\log(8/\beta)\rceil)}\}$
\State \Return $\mathbf{\hat{w}}\gets \mathcal{M}_E(D,-\ell,\mathcal{W})$
\EndProcedure
\vspace{0.5cm}
\Procedure{$\mathcal{A}_{\textup{Noise-GD}}$}{$D, \ell, (\eps',\delta'), \mathcal{C}$}
\State Noise variance $\sigma^2 \gets \frac{32L^2n^2\log(n/\delta')\log(1/\delta')}{\eps'^2}$, where $L$ is the Lipschitz constant of $\ell$.
\State Learning rate function $\eta: [n^2] \rightarrow \mathbb{R}$: $\eta(t) = \frac{\|\mathcal{C}\|_2}{\sqrt{t(n^2L^2+m\sigma^2)}}$.
\State Choose a point from $\mathcal{C}$, $\mathbf{w}_1$.
\For{$t=1$ to $n^2-1$}
\State Pick $(\mathbf{x},y) \sim_u D $ with replacement.
\State $\mathbf{w}_{t+1} \gets \Pi_{\mathcal{C}} \left( \mathbf{w}_t - \eta(t) \left[ n\nabla \ell(\mathbf{w}_t ; (\mathbf{x},y)) + b_t\right]\right)$, where $b_t \sim \mathcal{N}(0, \mathbb{I}_m\sigma^2)$.
\EndFor
\State \Return $\mathbf{w}_{n^2}$.
\EndProcedure
\end{algorithmic}
\end{algorithm}

To achieve a high-probability guarantee, algorithm $\mathcal{F}$ runs $\mathcal{A}_{\textup{Noise--GD}}$ $\lceil\log(8/\beta)\rceil$ times, with privacy parameters $\eps/\lceil\log(8/\beta)\rceil$ and $\delta/\lceil\log(8/\beta)\rceil$, and uses the Exponential Mechanism to pick the best hypothesis $\mathbf{\hat{w}}$, as described in Appendix D of \citet{BST14}.

\subsection{Proof of Lemma~\ref{lem:G_A}}

We state Lemma~\ref{lem:G_A} again for convenience.
\begin{lemma}[Lemma~\ref{lem:G_A}]\label{lem:G_A1}
For every given matrix $A$, let $\mathcal{G}_A \subseteq \mathcal{X}\times\{\pm1\}$ be the set of data points $(\mathbf{x},y)$ that satisfy the following two statements:
\begin{enumerate}[label=(\roman*)]
\item $\left\vert \|A\mathbf{x}\|^2_2 - \|\mathbf{x}\|^2_2\right\vert \leq \frac{\gamma}{100} \|\mathbf{x}\|^2_2$ and 
\item $\mathbf{w}_A^* = \frac{A\mathbf{w}^*}{\|A\mathbf{w}^* \|_2}$ has margin at least $96\gamma/100$ on $(\mathbf{x}_A,y)$, i.e.~$y \cdot \langle \mathbf{w}^*_A, \mathbf{x}_A\rangle \geq 96\gamma/100$.
\end{enumerate}
It holds that \[\Pr\limits_{(\mathbf{x},y)\sim D}[(\mathbf{x},y)\in \mathcal{G}_A] \geq 1-4\beta_{\mathit{JL}}.\]
\end{lemma}

\begin{proof}[Proof of Lemma~\ref{lem:G_A}]
By Lemma~\ref{lemma:DJL},
\begin{equation*}
\left\vert \|A\mathbf{u}\|^2_2 - \|\mathbf{u}\|^2_2\right\vert \leq \frac{\gamma}{100} \|\mathbf{u}\|^2_2
\end{equation*}
holds for a point $\mathbf{u}\in\mathbb{R}^d$ with probability at least $1-\beta_{\mathit{JL}}$. By union bound, it holds simultaneously for all points $\mathbf{x}+\mathbf{w}^*$, $\mathbf{x}-\mathbf{w}^*$, $\mathbf{x}$, and $\mathbf{w}^*$, with probability at least $1-4\beta_{\mathit{JL}}$. Under this condition, statement $(i)$ is true and for $y=1$ we have:
\begin{align*}
\langle \mathbf{w}^*, \mathbf{x}\rangle
& = \frac{1}{4}\|\ \mathbf{x}+\mathbf{w}^*\|^2_2 - \frac{1}{4}\|\ \mathbf{x}-\mathbf{w}^*\|^2_2 \\
& \leq \frac{1}{4\left(1-\frac{\gamma}{100}\right)}\|\ A(\mathbf{x}+\mathbf{w}^*)\|^2_2 - \frac{1}{4\left(1+\frac{\gamma}{100}\right)}\|\ A(\mathbf{x}-\mathbf{w}^*)\|^2_2 \\
& = \frac{1}{4\left(1-\frac{\gamma^2}{100^2}\right)} \left[\left(1+\frac{\gamma}{100}\right)\| A(\mathbf{x}+\mathbf{w}^*)\|^2_2 - \left(1-\frac{\gamma}{100}\right)\| A(\mathbf{x}-\mathbf{w}^*)\|^2_2 \right]\\
& = \frac{1}{1-\frac{\gamma^2}{100^2}} \left(\frac{1}{4}\|\ A\mathbf{x}+A\mathbf{w}^*\|^2_2 - \frac{1}{4}\|\ A\mathbf{x}-A\mathbf{w}^*\|^2_2 \right) \\
& \hspace{1.5cm} + \frac{\frac{\gamma}{100}}{4\left(1-\frac{\gamma^2}{100^2}\right)} \left(\|\ A\mathbf{x}+A\mathbf{w}^*\|^2_2+\|\ A\mathbf{x}-A\mathbf{w}^*\|^2_2\right)\\ 
& = \frac{1}{1-\frac{\gamma^2}{100^2}} \langle A\mathbf{w}^*, A\mathbf{x}\rangle + \frac{\frac{\gamma}{100}}{2\left(1-\frac{\gamma^2}{100^2}\right)} \left(\|\ A\mathbf{x}\|^2_2+\|\ A\mathbf{w}^*\|^2_2\right)\\
&\leq \frac{1}{1-\frac{\gamma^2}{100^2}} \langle A\mathbf{w^*}, A\mathbf{x}\rangle + \frac{\frac{\gamma}{100}\left(1+\frac{\gamma}{100}\right)}{1-\frac{\gamma^2}{100^2}} 
\end{align*}

Equivalently $\langle A\mathbf{w}^*, A\mathbf{x}\rangle  \geq \left(1-\frac{\gamma^2}{100^2}\right)\langle \mathbf{w^*}, \mathbf{x}\rangle - \frac{\gamma}{100}\left(1+\frac{\gamma}{100}\right)$.  Since $y=1$ and $\langle \mathbf{w^*}, \mathbf{x}\rangle = y\cdot \langle \mathbf{w^*}, \mathbf{x}\rangle \geq \gamma$, it follows that:
\[y\cdot \langle A\mathbf{w}^*, A\mathbf{x}\rangle \geq \left(1-\frac{\gamma^2}{100^2}\right)\gamma - \frac{\gamma}{100}\left(1+\frac{\gamma}{100}\right) \geq \frac{98\gamma}{100}\]
Therefore, for $y=1$, \[y\cdot \langle \mathbf{w}^*_A, \mathbf{x}_A \rangle = y\cdot \left\langle \frac{A\mathbf{w}^*}{\|A\mathbf{w}^*\|_2}, \frac{A\mathbf{x}}{\|A\mathbf{x}\|_2} \right\rangle \geq \frac{\frac{98\gamma}{100}}{1+\frac{\gamma}{100}} \geq \frac{96\gamma}{100}.\]

The proof for $y=-1$ is similar. We conclude that with probability at least $1-4\beta_{\mathit{JL}}$, statements $(i)$ and $(ii)$ are true.
\end{proof}

\section{Proof of Sample Complexity: $(\eps,0)$-DP}\label{app:pure}
\begin{theorem}[Sample Complexity, Theorem~\ref{th:UB_pure}]\label{th:UB_pure1}
Algorithm $\mathcal{A}_{\alpha,\beta,\eps,\gamma}$ is an $(\alpha, \beta, \gamma)$-learner with sample complexity
\[n=\frac{1}{\alpha\eps\gamma^2}\cdot \mathrm{polylog}\left(\frac{1}{\alpha}, \frac{1}{\beta}, \frac{1}{\eps}, \frac{1}{\gamma}\right).\]
\end{theorem}

\begin{proof}[Proof of Theorem \ref{th:UB_pure}]
As in the previous section, the first step of the algorithm is to sample matrix $A$ uniformly at random from $U=\left\{\pm\frac{1}{\sqrt{m}}\right\}^{m\times d}$. From Lemma~\ref{lem:G_A}, it follows that:
\begin{equation*}
\E\limits_{A}\left[\Pr\limits_{(\mathbf{x},y)\sim D}[(\mathbf{x},y)\notin \mathcal{G}_A]\right] \leq 4\beta_{\mathit{JL}}
\end{equation*}
And, by  Markov's inequality,
\begin{align*}
\Pr\limits_A\left[ \Pr\limits_{(\mathbf{x},y)\sim D}[(\mathbf{x},y)\notin \mathcal{G}_A] \geq \beta'\right] &\leq \frac{\E\limits_{A}\left[\Pr\limits_{(\mathbf{x},y)\sim D}[(\mathbf{x},y)\notin \mathcal{G}_A]\right]}{\beta'}
\leq \frac{4\beta_{\mathit{JL}}}{\beta'}.
\end{align*}

We set $\beta'=\alpha\beta/4n$. Then, substituting $\beta_{\mathit{JL}}=\frac{\alpha\beta^2}{64n}$, we get that with probability at least $1-\beta/4$,
\begin{equation*}\label{eq:prob_G_A1}
\Pr\limits_{(\mathbf{x},y)\sim D}[(\mathbf{x},y)\in \mathcal{G}_A] \geq 1-\beta'.
\end{equation*}

Therefore, with probability $1-\beta/4$, the sampled matrix $A$ satisfies the above inequality, that is, a point $(\mathbf{x},y)\sim D$ is in $\mathcal{G}_A$ with probability at least $1-\beta'$.
Furthermore, by union bound, $\forall (\mathbf{x},y)\in S$ it holds that $(\mathbf{x},y) \in \mathcal{G}_A$, with probability at least $1-n\beta'\geq 1-\beta/4$.

For the remainder of the proof, we condition on the event that
\begin{enumerate}
\item $\Pr\limits_{(\mathbf{x},y)\sim D}[(\mathbf{x},y)\in \mathcal{G}_A] \geq 1-\beta'$ holds for $A$ and 
\item $S \subseteq \mathcal{G}_A$, that is, $\mathbf{w}_A^*$ has margin at least $96\gamma/100$ on $S_A$.
\end{enumerate}
This event occurs with probability at least $1-\beta/4-\beta/4=1-\beta/2$.

\begin{claim}\label{cl:emp_error_pure}
If $n=\frac{1}{\alpha\eps\gamma^2}\cdot \mathrm{polylog}\left(\frac{1}{\alpha}, \frac{1}{\beta},  \frac{1}{\eps}, \frac{1}{\gamma}\right)$, then with probability $1-\beta/4$, for hypothesis $\mathbf{\hat{w}}$ returned by the Exponential Mechanism it holds that
\begin{equation}\label{eq:emp_risk_pure_a}
\frac{1}{n}\sum_{(\mathbf{x}_A,y)\in S_A}\mathbbm{1}\left\{y\cdot \langle \mathbf{\hat{w}}, \mathbf{x}_A\rangle < \frac{\gamma}{10}\right\} \leq \frac{\alpha}{4}.
\end{equation}
\end{claim}

\begin{proof}[Proof of Claim~\ref{cl:emp_error_pure}]
Every point in $\mathcal{B}_2^m$ is within $\gamma/10$ from a center of $\mathcal{W}$.
Let $\mathbf{w}_c^*$ be the center within $\gamma/10$ from $\mathbf{w}^*_A$, that is, 
\begin{equation}\label{eq:w_center}
\|\mathbf{w}_A^* - \mathbf{w}_c^*\|_2 \leq \gamma/10.
\end{equation}

Recall that in our conditioned probability space, 
\begin{equation}\label{eq:w_star_A}
y\cdot\langle \mathbf{w}^*_A, \mathbf{x}_A \rangle \geq 96\gamma/100
\end{equation}
holds for all $(\mathbf{x}_A,y)\in S_A$.
Therefore, for all $(\mathbf{x}_A,y)\in S_A$, 
\begin{align*}
y\cdot\langle \mathbf{w}^*_c, \mathbf{x}_A \rangle 
&= y\cdot\langle \mathbf{w}^*_A, \mathbf{x}_A \rangle - y\cdot\langle \mathbf{w}^*_A - \mathbf{w}^*_c, \mathbf{x}_A \rangle \\
&\geq y\cdot\langle \mathbf{w}^*_A, \mathbf{x}_A \rangle - \|\mathbf{w}_A^* - \mathbf{w}_c^*\|_2\cdot \|\mathbf{x}_A\|_2 \\
&\geq 96\gamma/100 - \gamma/10 = 86\gamma/10 > \gamma/10. \tag{by inequalities~\eqref{eq:w_center}, \eqref{eq:w_star_A}}
\end{align*}

It follows that 
\begin{equation}\label{eq:opt_exp}
\max\limits_{\mathbf{w}\in\mathcal{W}} u(S_A,\mathbf{w}) \geq u(S_A, \mathbf{w}^*_c) = -\frac{1}{n}\sum_{(\mathbf{x}_A,y)\in S_A}\mathbbm{1}\left\{y\cdot \langle \mathbf{w}, \mathbf{x}_A\rangle < \frac{\gamma}{10}\right\} = 0.
\end{equation}

By Lemma~\ref{lem:expmech} and inequality~\eqref{eq:opt_exp}, with probability at least $1-\beta/4$, it holds that:
\begin{equation}\label{eq:emp_risk_pure}
\frac{1}{n}\sum_{(\mathbf{x}_A,y)\in S_A}\mathbbm{1}\left\{y\cdot \langle \mathbf{\hat{w}}, \mathbf{x}_A\rangle < \frac{\gamma}{10}\right\} \leq \frac{2}{n\eps}(\ln(|\mathcal{W}|) + \ln(4/\beta))
\end{equation}

It is a well-known result that the covering number of an $m$-dimensional unit ball by balls of radius $\gamma/10$ is at most $O\left(\left(\frac{1}{\gamma/10}\right)^m\right)$. Therefore, substituting $m=O\left(\frac{\log(n/\alpha\beta)}{\gamma^2}\right)$, it follows that 
\begin{equation*}\label{eq:coverno}
\ln|\mathcal{W}| = \frac{1}{\gamma^2}\cdot \mathrm{polylog}\left(n,\frac{1}{\alpha}, \frac{1}{\beta}, \frac{1}{\gamma}\right).
\end{equation*} 

Thus, by inequality~\eqref{eq:emp_risk_pure}, if $n=\frac{1}{\alpha\eps\gamma^2}\cdot\mathrm{polylog}\left(\frac{1}{\alpha}, \frac{1}{\beta}, \frac{1}{\gamma}, \frac{1}{\eps}\right)$ then with probability at least $1-\beta/4$,
\begin{equation*}
\frac{1}{n}\sum_{(\mathbf{x}_A,y)\in S_A}\mathbbm{1}\left\{y\cdot \langle \mathbf{\hat{w}}, \mathbf{x}_A\rangle < \frac{\gamma}{10}\right\} \leq \frac{\alpha}{4}.
\end{equation*}
This concludes the proof of the claim. \end{proof}

\begin{claim}\label{cl:generalization_pure}
If $n=\frac{1}{\alpha\eps\gamma^2}\cdot \mathrm{polylog}\left(\frac{1}{\alpha}, \frac{1}{\beta}, \frac{1}{\eps}, \frac{1}{\gamma}\right)$, then with probability $1-\beta/2$, the error of the returned classifier $\mathbf{\hat{w}}^\top A$ on distribution $D$ is 
\begin{equation}\label{eq:generalization_pure}
\Pr\limits_{(\mathbf{x},y)\sim D}[y\cdot \langle \mathbf{\hat{w}}^\top A,\mathbf{x}\rangle < 0] \leq \alpha.
\end{equation}
\end{claim}

\begin{proof}[Proof of Claim~\ref{cl:generalization_pure}]
Let $D_A$ denote the probability distribution with domain $\mathcal{B}_2^{m} \times \{\pm1\}$, from which a sample $(\mathbf{x}_A, y)\in S_A$ is drawn. Let us also denote by $D_{|\mathcal{G}_A}$ distribution $D$ restricted on $\mathcal{G}_A$. In our conditioned probability space, $S_A \sim D_A^n$, where the probability density function of $D_A$ would be defined as \[\Pr\limits_{(\mathbf{x}_A,y)\sim D_A}[\mathbf{x}_A = \mathbf{x}' \wedge y=y'] = \Pr\limits_{(\mathbf{x},y)\sim D_{|\mathcal{G}_A}} \left[\frac{A\mathbf{x}}{\|A\mathbf{x}\|_2}=\mathbf{x}' \wedge y= y'\right].\]

Let $\mathcal{H}=\left\{ h:\{\mathbf{x}_A \mid (\mathbf{x},y)\in \mathcal{G}_A\} \rightarrow \{\pm1\} \text{ s.t. } h(\mathbf{x})=\mathrm{sign}(\langle \mathbf{w}, \mathbf{x}\rangle) \text{ for some } \mathbf{w}\in\mathcal{B}_2^m\right\}$ be a concept class of threshold functions in $\mathcal{B}_2^m$. By Theorem 3.4 of~\cite{AB09}, $\mathrm{VCdim}(\mathcal{H})=m+1$.

By the generalization bound of Theorem 5.7 of \cite{AB09}, stated in Lemma~\ref{th:AB09}, it holds that:
\begin{equation*}
\Pr_{S_A\sim D_A^n} \left[\exists h\in \mathcal{H}: \mathrm{err}_{D_A}(h) > 2\cdot \frac{1}{n}\sum_{(\mathbf{x}_A,y)\in S_A}\mathbbm{1}\left\{h(\mathbf{x}_A) \neq y\right\} + \frac{\alpha}{4}\right] \leq 4\Pi_{\mathcal{H}}(2n)\exp(-\alpha n/32)
\end{equation*}
where the growth function $\Pi_{\mathcal{H}}(2n) \leq (2n)^{m+1}+1$, by Theorem 3.7 of  \cite{AB09}.

Using $m=O\left(\frac{\log(1/\alpha\beta)}{\gamma^2}\right)$, it holds that if $n=\frac{1}{\alpha\gamma^2} \cdot \mathrm{polylog}\left(\frac{1}{\alpha}, \frac{1}{\beta}, \frac{1}{\gamma}\right)$ then $4\Pi_{\mathcal{H}}(2n)\exp(-\alpha n/32) \leq \beta/4$. Therefore, with probability at least $1-\beta/4$,
\begin{equation}\label{eq:general_pure}
\mathrm{err}_{D_A}(f_{\mathbf{\hat{w}}}) \leq 2\cdot \frac{1}{n}\sum_{(\mathbf{x}_A,y)\in S_A}\mathbbm{1}\left\{y\cdot\langle \mathbf{\hat{w}}, \mathbf{x}_A \rangle<0\right\} +\frac{\alpha}{4}
\end{equation}

Using $m=O\left(\frac{\log(1/\alpha\beta)}{\gamma^2}\right)$, it holds that if $n=\frac{1}{\alpha\gamma^2} \cdot \mathrm{polylog}\left(\frac{1}{\alpha}, \frac{1}{\beta}, \frac{1}{\gamma}\right)$ then $4\Pi_{\mathcal{H}}(2n)\exp(-\alpha n/32) \leq \beta/4$. Therefore, with probability at least $1-\beta/4$,
\begin{equation}\label{eq:general_pure}
\mathrm{err}_{D_A}(f_{\mathbf{\hat{w}}}) \leq 2\cdot \frac{1}{n}\sum_{(\mathbf{x}_A,y)\in S_A}\mathbbm{1}\left\{y\cdot\langle \mathbf{\hat{w}}, \mathbf{x}_A \rangle<0\right\} +\frac{\alpha}{4}
\end{equation}

By Claim~\ref{cl:emp_error_pure}, $\frac{1}{n}\sum\limits_{(\mathbf{x}_A,y)\in S_A}\mathbbm{1}\left\{y\cdot\langle \mathbf{\hat{w}}, \mathbf{x}_A \rangle<0\right\} \leq \frac{1}{n}\sum\limits_{(\mathbf{x}_A,y)\in S_A}\mathbbm{1}\left\{y\cdot\langle \mathbf{\hat{w}}, \mathbf{x}_A \rangle<\frac{\gamma}{10}\right\} \leq \frac{\alpha}{4}$ holds with probability $1-\beta/4$, if $n=\frac{1}{\alpha\eps\gamma^2}\cdot \mathrm{polylog}\left(\frac{1}{\alpha}, \frac{1}{\beta}, \frac{1}{\delta}, \frac{1}{\eps}, \frac{1}{\gamma}\right)$. 

Therefore, by inequality~(\ref{eq:general_pure}), if $n=\frac{1}{\alpha\eps\gamma^2}\cdot \mathrm{polylog}\left(\frac{1}{\alpha}, \frac{1}{\beta}, \frac{1}{\delta}, \frac{1}{\eps}, \frac{1}{\gamma}\right)$, then with probability at least $1-\beta/4-\beta/4=1-\beta/2$,
\[\mathrm{err}_{D_A}(f_{\mathbf{\hat{w}}}) \leq 2\cdot\frac{\alpha}{4}+\frac{\alpha}{4}=\frac{3\alpha}{4}.\]

Equivalently, with probability at least $1-\beta/2$, 
\[\Pr\limits_{(\mathbf{x},y)\sim D_{|\mathcal{G}_A}}[y\cdot \langle \mathbf{\hat{w}}^\top A,\mathbf{x}\rangle < 0] = \Pr\limits_{(\mathbf{x}_A,y)\sim D_{A}}[y\cdot \langle \mathbf{\hat{w}},\mathbf{x}_A\rangle < 0]  = \mathrm{err}_{D_A}(f_{\mathbf{\hat{w}}})  \leq \frac{3\alpha}{4}.\]

Since, by Condition 1., $\Pr\limits_{(\mathbf{x},y)\sim D}[(\mathbf{x},y)\notin \mathcal{G}_A] \leq \beta' \leq \frac{\alpha}{4}$, it follows that with probability at least $1-\beta/2$,
\begin{align*}
\Pr\limits_{(\mathbf{x},y)\sim D}[y\cdot \langle \mathbf{\hat{w}}^\top A,\mathbf{x}\rangle < 0] 
& \leq \Pr\limits_{(\mathbf{x},y)\sim D_{|\mathcal{G}_A}}[y\cdot \langle \mathbf{\hat{w}}^\top A,\mathbf{x}\rangle < 0] \cdot (1-\beta')+  1\cdot \beta'\\
&\leq \frac{3\alpha}{4} \cdot (1-\beta') + \beta' \\
& \leq \frac{3\alpha}{4}+\frac{\alpha}{4} \leq \alpha.
\end{align*}
This completes the proof of the claim.
\end{proof}

Accounting for the probability that we are not in the conditioned space, we conclude that if $n=\frac{1}{\alpha\eps\gamma^2}\cdot \mathrm{polylog}\left(\frac{1}{\alpha}, \frac{1}{\beta}, \frac{1}{\delta}, \frac{1}{\eps}, \frac{1}{\gamma}\right)$, then with probability at least $1-\beta/2-\beta/2=1-\beta$, $\mathrm{err}_{D}(f_{\mathbf{\hat{w}}^\top A}) \leq  \alpha$.  This completes the proof of the theorem.
\end{proof}

\end{document}